\documentclass{comsoc2010}

\usepackage{alltt,theorem,amssymb,amsmath,latexsym}
\usepackage[ruled,linesnumbered]{algorithm2e}
\usepackage{graphicx}

\usepackage{times}
\usepackage{epic}
\usepackage{wrapfig}
\input{epsf}

\newcommand{\comment}[1]{}

\newenvironment{program}{\fontsize{10pt}{12pt}\begin{alltt}}
  {\end{alltt}}
 \newcommand{\keywordfont}[1]{\textrm{\textit{\textbf{#1}}}}
\newcommand{\xwhile}{\keywordfont{while}}

\newcommand{\xif}{\keywordfont{if}}

\newcommand{\xelse}{\keywordfont{else}}

\newtheorem{theorem}{Theorem}[section]
\newtheorem{definition}{Definition}
\newtheorem{observation}{Observation}
\newtheorem{proposition}{Proposition}

\newenvironment{proof}[1][Proof]{\begin{trivlist}
\item[\hskip \labelsep {\bfseries #1}]}{\end{trivlist}}

\newcommand{\zuck}{\texttt{REVERSE}}
\newcommand{\lslg}{\texttt{LSLG}}
\newcommand{\lsla}{\texttt{LSLA}}

\title{An Empirical Study of Borda Manipulation}
\author{Jessica Davies, George Katsirelos, Nina Narodytska and Toby Walsh}

\begin{document}

\begin{abstract}
We study the problem of coalitional manipulation in elections using
the unweighted Borda rule. We provide empirical evidence of the
manipulability of Borda elections in the form of two new greedy
manipulation algorithms based on intuitions from the bin-packing and
multiprocessor scheduling domains. Although we have not been able to show that
these algorithms beat existing methods in the worst-case,
our empirical evaluation shows that they significantly
outperform the existing method and are able to find
optimal manipulations in the vast majority of the
randomly generated elections that we tested. ÊThese
empirical results provide further evidence that the Borda rule
provides little defense against coalitional manipulation.
\end{abstract}

\section{Introduction}

Elections are a well established mechanism to aggregate the preferences of individuals to reach a consensus decision. New applications of voting and social choice have emerged in the field of multiagent systems and are used on a daily basis by many people in the form of polls and ratings systems on the internet.  As an election is meant to be a fair way of reaching a decision, it is important to study the weaknesses of different voting systems with respect to their vulnerability to manipulation, bribery and control. In this paper we focus on the manipulation problem, where a coalition of agents votes to ensure a desired outcome rather than reporting their true preferences.  
It is assumed that the manipulators act with full knowledge of the votes of the remaining electorate, but even so, the structure of the voting system may make it difficult to ensure that the desired candidate wins.  No practical voting system can prevent a coalition of enough manipulators from achieving their goal in all elections.  However, some mechanisms may be easier to manipulate than others.  For example, the required size of the coalition may be impractical, especially in real-world settings where obtaining the cooperation of and coordinating more than two or three people can be difficult.   Even if the number of extra votes isn't a concern, calculating the required set of manipulator votes may be computationally infeasible.\\
 
In this work we study the voting system based on using the Borda rule to aggregate the votes. The Borda rule is a positional scoring rule proposed by the French scientist Jean-Charles de Borda in 1770.  Like all positional scoring rules, each voter simply ranks the $m$ candidates according to their preference.  The votes are aggregated by adding a score of $m-k$ to a candidate for each time it appears $k^{th}$ in a vote.  The candidates with the highest aggregated score win the election.  The simplicity of this rule may have contributed to its independent reinvention on at least one other occasion; political elections in two Pacific island states use slight modifications of the Borda rule~\cite{reillyIPR02}.  It is also commonly used in competitions such as the Eurovision song contest, the election of the Most Valuable Player in major league baseball, and the Robocup competition.\\

The susceptibility of Borda elections to manipulation has been strongly suggested by recent theoretical work.   Although the problem is NP-hard if the manipulators' votes are weighted~\cite{conitzerJACM08}, in the unweighted case the complexity class is still frustratingly unknown.  Xia et al. observe that:

\begin{quote}
{\em ``The exact complexity of the problem [coalition 
manipulation with unweighted votes] 
is now known with respect to almost all of the prominent
voting rules, with the glaring exception of Borda''}
\cite{xiaEC10}
\end{quote}

A number of recent theoretical results suggest that manipulation may often be computational easy~\cite{csaaai2006,prjair07,xcec08,xcec08b}. 
Brelsford et al.~\cite{brelsfordAAAI08} showed that weighted (and unweighted) Borda manipulation has a FPTAS, which means that finding a very close to optimal manipulation can be done in polynomial time.  Along these lines, Zuckerman et al.~\cite{zuckermanSODA08} gave a simple greedy algorithm to calculate a manipulation, that in the unweighted case uses at most one more manipulator than is optimal.  In addition, even Borda himself appears to
have recognised that his rule was susceptible to manipulation, having retorted that:

\begin{quote}
{\em ``My scheme is intended only for honest men''},
quoted on page 182 of \cite{black58}
\end{quote}

More recently, strategic voting was identified in the 1991 presidential candidate elections in the Republic of Kiribati (where a variant of the Borda rule is used) 
 \cite{reillyIPR02}. This suggests that the manipulability of the Borda rule is not just a theoretical possibility but a practical reality.\\

The manipulability of voting rules has also been studied empirically~\cite{walshIJCAI09,walshECAI10}.  For example, Walsh studied the Single Transferable Vote rule, which is theoretically NP-hard to manipulate.  However, he provided ample evidence that in practise, elections using this rule are easy to manipulate~\cite{walshECAI10}. We provide further empirical evidence that the Borda rule provides little defense to manipulation, by showing that in many elections, an optimal manipulation can be found (and often verified) in polynomial time.  Our starting point is the greedy algorithm of Zuckerman et al.~\cite{zuckermanSODA08}, which decides the vote of each manipulator in turn by reversing the candidates ordered by current score.  Although this algorithm provides a guarantee that in the worst case it only uses one more manipulator than is optimal, the theoretical analysis does not extend to answer the question of how frequently it uses this extra manipulator.  Perhaps another greedy algorithm exists that finds the optimal manipulation much more frequently. If so, it could be used in conjunction with that of Zuckerman et al. to provide a verified optimal solution whenever it finds a solution using one fewer manipulator.  We introduce two new greedy algorithms, based on intuitions from the bin-packing and multiprocessor scheduling domains, and provide theoretical and empirical comparison between their performance and that of Zuckerman et al.'s greedy algorithm.  The new algorithms result in a significant improvement over Zuckerman et al.'s algorithm, allowing the optimal manipulation to be found and verified quickly on 99\% of more than 60,000 randomly generated elections.\\

The paper continues with the definitions and background in Section~\ref{sec:background}, followed in Section~\ref{sec:algs} by our new greedy algorithms. Section~\ref{sec:exper} presents the experimental results and we conclude in the last section.

\section{Background}
\label{sec:background}
In this section we introduce notation and definitions that will be used throughout the paper.\\

An \emph{election} is a pair $E = (V,m)$  where $m$ is the number of candidates.  We refer to the distinguished candidate who the manipulators want to win the election as candidate $1 \leq d \leq m$;  the other $m-1$ candidates are then the \emph{competing} candidates.  $V$ is a set of \emph{votes}, where a vote is an ordering of the candidates $v = c_1 > c_2>...> c_m$ such that $\bigcup c_j = \{1,..,m\}$.  Given a vote $v$, the \emph{score} of a candidate $i$ under the Borda rule, denoted $s(v, i)$, equals $m-k$ where $c_k = i$.   If $V$ is a set of votes, then the score of a candidate $i$ given by these votes is $s(V, i) = \Sigma_{v\in V}s(v, i)$.  Given an election $E=(V,m)$, the \emph{winners} are defined as those candidates $1\leq i \leq m$ such that $s(V, i)$ is maximal.  A \emph{manipulation} of an election $E = (V,m)$ is a set of manipulator votes $M$ such that $s(V\cup M, d) \geq s(V\cup M, i)$ for all $i \neq d$. We assume that ties are broken in favour of the manipulators. The manipulation problem is to find a manipulation such that $|M| = n$ is minimized.  Sometimes we will refer to a manipulation using $n$ votes as an \emph{n-manipulation}.\\

We define some additional notation that will be helpful in describing our greedy algorithms.  

\begin{definition}
\label{def:gap}
Given an election $E = (V,m)$, a number of manipulators $n$, the \emph{gap} of candidate $1\leq i \leq m$, is defined as $g_{E,n}(i) = s(V, d) + n(m-1) - s(V, i)$.  If the context is clear, we call the gap of candidate $i$ simply $g_i$.
\end{definition}
Intuitively, the gap of a candidate $i$ is the difference between the score the distinguished candidate receives after the manipulators have voted, and the score of $i$ before the manipulators vote. Without loss of generality, we assume that the manipulators always rank $d$ first. Note that if $g_i$ is negative for any $i$, then there is no $n$-manipulation.

\begin{definition}
\label{def:matrix}
Given an election $E = (V, m)$, an $n$-manipulation matrix $A_{E,n}$ is an $n\times m$ matrix such that all elements of column $d$ are equal to $m-1$, each row contains all numbers from $0$ to $m-1$ and column $i$ sums to at most $g_{E,n}(i)$ for all $1\leq i \leq m$.  
\end{definition}
It is easy to see that such a matrix represents an $n$-manipulation of the election, where each column represents a competing candidate, and each row corresponds to the vote of a distinct manipulator.  We will drop the parameters $E$ and $n$ and refer to matrix $A$ when the context is clear.  We use the notation $A(i)$ to denote the $i^{th}$ column of $A$, and $sum(A(i))$ is defined to be the sum of the elements in $A(i)$.

\begin{observation}
\label{obs:1}
Given an election $E = (V,m)$ and a number of manipulators $n$, if $\Sigma_{i=1}^{m-1}g_{E,n}(i) < (n/2)(m-1)(m-2)$ then there is no $n$-manipulation.
\end{observation}
This follows directly from Definition~\ref{def:matrix}, since each of the $n$ manipulator votes contributes a total of $\Sigma_{k=0}^{m-2}k = (1/2)(m-1)(m-2)$ score to the scores of the competing candidates. In other words, there must be enough difference between the original scores of the competing candidates and the achievable score of the distinguished candidate, otherwise an $n$-manipulation can not exist.  We call the multiset containing $n$ copies of each $0\leq k\leq m-2$ $S_n$.\\

The greedy algorithm of Zuckerman et al.~\cite{zuckermanSODA08} is shown in Figure~\ref{fig:zuck}, and from now on will be referred to as $\zuck$.  The manipulation matrix $A$ starts off empty, and is augmented row by row until enough manipulators have been added that the distinguished candidate wins.  The $\texttt{sort}$ procedure puts the distinguished candidate first, and then sorts the competing candidates in increasing order by their current score, in order to create the next manipulator's vote.\\

\noindent \emph{Example 1}.  Suppose $E = (V,5)$ where $V$ contains the votes $v_1 = 1>2>3>4>5$, $v_2 = 2>3>4>1>5$, $v_3 = 3>4>1>2>5$ and $v_4 = 4>1>2>3>5$, and $d=5$.  Then $s(V, 5) = 0$, and $s(V, i) = 10$ for all competing candidates $i < 5$.  In order for candidate $5$ to win the election, at least $4$ manipulators are required since $\Sigma_i g_{E,3}(i) = 4*(4*3 - 10) = 8$ but $(n/2)(m-1)(m-2) = 1.5*4*3 = 18$.  $\zuck$ will make the first manipulator vote $w_1 = 5 > 1 > 2 > 3> 4$ (ordering the competing candidates arbitrarily), at which point, e.g., $s(V\cup \{w_1\}, 1) = 10 + 3 = 13$.  The candidates' scores are shown in Figure~\ref{fig:ex1} after each iteration of the \keywordfont{while} loop.  Since $s(V\cup\{w_1,w_2,w_3,w_4 \}, 5) = 16$, $\zuck$ finds the optimal manipulation.

\begin{figure*}
\begin{program}
REVERSE(V,m,d)
1. A[i] \(\leftarrow \emptyset\) for all 1\(\leq\)i\(\leq\)m 
2. n \(\leftarrow\) 0
3. \xwhile max\(_i\)\(\{\)sum(A[i]) + s(V,i)\(\}\) \(>\) sum(A[d]) + s(V,d)
4.    w \(\leftarrow\) sort\(\{\)i < j\(\iff\)(sum(A[i])+s(V,i) < sum(A[j])+s(V,j) or i=d)\(\}\)
5.    A[i].push(s(w,i)) for all i
6.    n \(\leftarrow\) n + 1      
7. return A
\end{program}
\caption{\label{fig:zuck} The greedy algorithm of Zuckerman et al.~\cite{zuckermanSODA08}.}
\end{figure*}

\begin{figure*}
\centering{
\begin{tabular}{r|ccccc}
Candidate $i$ & 1 & 2 & 3 & 4 & 5\\
\hline
$s(V,i)$ & 10 & 10 & 10 & 10 & 0\\
$s(V\cup \{w_1\}, i)$ & 13 & 12 & 11 & 10 & 4\\
$s(V\cup \{w_1, w_2\}, i)$ & 13 & 13 & 13 & 13 & 8\\
$s(V\cup\{w_1,w_2, w_3\}, i)$ & 16 & 15 & 14 & 13 & 12\\
$s(V\cup\{w_1,w_2,w_3, w_4\},i)$ & 16 & 16 & 16 & 16 & 16\\
\end{tabular}}
\caption{ \label{fig:ex1} Scores given by $\zuck$, for Example 1.}
\end{figure*}

\section{Greedy Algorithms for Borda Manipulation}
\label{sec:algs}
The definition of manipulation matrix from Section~\ref{sec:background} is a useful abstraction, that suggests a connection to bin-packing or multiprocessor scheduling~\cite{coffman06}.  Intuitively, the elements of the manipulators votes, $S_n$, must be assigned to the columns of $A$ such that the sum of each column is at most $g_i$.  In the bin-packing problem, a set of objects with sizes between zero and one must be grouped into a minimum number of bins such that the sum of the objects in each bin is at most one.  So in our case, the set of objects would be $S_n$, representing the elements whose positions in manipulation matrix $A$ are initially unknown.  One of the main differences is that our matrix $A$ has a constraint on each row, that it must contain all values from $0$ to $m-1$, and it is not clear how this translates to other domains.  Luckily, Theorem~\ref{thm:1} tells us that we don't have to worry about this constraint.  If a correctly sized matrix $B$ containing $n$ elements equal to $j$ for each $0\leq j\leq m-1$ can be found such that the column sums are at most the candidate's gaps and column $d$ contains all the $m-1$'s, then it can always be converted to a manipulation matrix $A$.

\begin{theorem}
\label{thm:1}
Suppose there exists an $n \times m$ matrix $B$ such that the total number of elements in $B$ equal to $k$, for each $0\leq k\leq m-1$ is $n$. Let the sum of the elements in the $i^{th}$ column of $B$ be $g_i$. 
Then there is another $n\times m$ matrix $A$ with the same set of elements as $B$ and the same column sums, such that each row contains exactly one element equal to $k$, for each $0\leq k\leq m-1$.
\end{theorem}
\begin{proof}
By induction on $n$.  When $n=1$, we have $B = [b_{1,1},...,b_{1,m}]$ such that $B$ contains exactly one element of value $k$ for each $0\leq k\leq m-1$.  Therefore, just set $A=B$.  

Assume that the theorem holds for all numbers of rows less than $n$.  We prove that it also holds for $n$ rows.  Let $B$ be an $n\times m$ matrix such that the total number of elements in $B$ equal to $k$, for each $0\leq k\leq m-1$ is $n$. Let the sum of the elements in the $i^{th}$ column be $g_i$. 

Define a bipartite graph $G = (S\cup T,E)$ such that the set of left-hand vertices is $S = \{0,...,m-1\}$ (these will represent the set of values of the elements of row 1 in $A$), and the set of right-hand vertices is $T =\{1,...,m\}$ representing the columns of $B$.  $E$ contains an edge $(i,j)_k$ for each $i\in S$, $j\in T$ and $1\leq k\leq n$ such that $i = B(k,j)$. 

Note that there can be up to $n$ edges between two vertices $i$ and $j$. Since every value appears $n$ times in $B$, $|\{(k,j) : i = B(k,j)\}| = n$ and so the degree of each $i\in S$ is exactly $n$.  For each $j\in T$, the degree will also be $n$: one edge to each $i = B(k,j)$, $1\leq k\leq n$. 

Therefore, if we take any $P\subseteq S$, $n|P|$ edges leave $P$.  Since every vertex in $T$ is also of degree $n$, each vertex in the neighbourhood of $P$, $nbhd(P)$, can accommodate at most $n$ incoming edges.  Therefore, $|nbhd(P)|$ is not less than $|P|$.  Since the Hall condition holds~\cite{hall}, there is a perfect matching in $G$ that assigns each value from $0$ to $m-1$ to a position in the first row of $B$, as follows.

Let $M = \{e_1,...,e_{m}\} \subseteq E$ be the set of edges in the matching.  For each $e = (i,j)_k\in M$, let $A(1,j) = i$.   Since $M$ is a matching, each $i$, $0\leq i\leq m-1$ appears in exactly one column, and each column is assigned exactly one element. Therefore, the first row of $A$ is well defined.  Also note that for each column $j$, $A(1,j)$ appears in the $j^{th}$ column of $B$.

Let $B'$ be the matrix defined by taking $B$ and removing one element equal to $A(1,j)$ from each column $j$.  Then $B'$ is an $n-1\times m$ matrix containing exactly $n-1$ elements equal to $i$ for each $0\leq i\leq m-1$, since the elements removed were one of each value.  The column sums for $B'$ are $g_j - A(1,j)$ for all columns $j$.  By the induction hypothesis, there exists an $n-1\times m$ matrix $A'$ such that $A'$ contains the same elements as $B'$ and the same column sums, but each row of $A'$ contains exactly one element equal to $i$, for $0 \leq i\leq m-1$.  Given that we've already defined the first row of $A$, let the remaining $n-1$ rows be $A'$.  Then $A$ contains the same set of values as $B$, with the same column sums $A(1,j) + (g_j - A(1, j)) = g_j$, and every row of $A$ contains exactly one element equal to $i$, for each $0\leq i\leq m-1$.

Therefore, by induction, the theorem holds for all $n$.
\hfill $\Box$
\end{proof}

If a matrix $B$ exists whose column sums are at most the value of the candidates' gaps, and sum($B[d]$) = $g_d$, then matrix $A$ gives a manipulation, where each row of $A$ defines the vote of one of the manipulators.  Therefore, we can devise algorithms to discover $B$ and be assured that $A$ exists.\\

However, the manipulation problem has two additional differences to bin-packing.  First, the number of objects in each bin must be exactly $n$, while bin-packing has no such constraint.  Secondly, each of our bins has a different maximum capacity $g_i$.  The former constraint has been studied in the multiprocessor scheduling domain, where the problem is to schedule jobs on a set of $n$ processors such that the memory resources are never exceeded and the time to complete all jobs is minimized~\cite{krauseJACM75}.  Our problem corresponds to the case where each job takes a unit of processing time.   For each element $a\in S_n$, there is a job with memory requirement equal to $a$.  The number of processors is equal to the number of manipulators $n$, and the amount of available memory resource at time step $i$ is equal to $g_i$.  We wish to find a schedule that uses $m-1$ time steps, which will be possible if an $n$-manipulation exists.  Krause et al. consider the case where the memory resource remains constant over time, and present theoretical analysis of a simple scheduling algorithm that assigns the jobs one at a time to particular time steps.  Their scheduler takes the unassigned job with largest memory requirements and assigns it to a time step (with at least one processor free), that has the maximum remaining available memory.  If no time step exists that can accommodate this job, a new time step is added.\\

Our first greedy algorithm is based on this same intuition, where it translates to giving the largest scores to the competing candidates that have the least score so far.  In this it is similar to $\zuck$, but we are now free to pursue this heuristic strictly, while $\zuck$ for example decides which candidate the second voter's $m-2$ should be assigned to \emph{after} the smaller scores of the first manipulator are assigned.  This can sometimes be an advantage, but it may also lead the algorithm to make more serious mistakes, as we will show.

\subsection{Largest Score in Largest Gap}
Our first greedy algorithm, $\lslg$ is shown in Figure~\ref{fig:lslg}.  $\lslg$ takes the number of manipulators as an argument and returns the matrix $B$ (from Theorem~\ref{thm:1}) if it is able to find an $n$-manipulation.  On line 1, the matrix $B$ (represented as an array of vectors) is initialized so that every column vector is empty.  On line 2, the column corresponding to the distinguished candidate is filled with the maximum value, $m-1$.  On line 3, the array $S$ is initialized with the sorted elements of $S_n$ defined in Section~\ref{sec:background}.  Each iteration of the $\keywordfont{while}$ loop on lines 4-7 removes the first (largest) element of $S$ and pushes it (on line 6) into the column of $B$ that has the lowest sum so far.  Note that we use the notation $|B(i)|$ to denote the current number of elements in the $i^{th}$ column of $B$. Once all elements of $S$ have been assigned, the loop terminates and line 8 checks if a valid manipulation has been produced.  If so, $B$ is returned, and if not, the algorithm reports Failure.

\begin{figure*}
\begin{program}
LSLG(V,n,d)
   // B[i] is the \(i^{th}\) column of B 
1. B[i] \(\leftarrow \emptyset\) for all 1\(\leq\)i\(\leq\)m
   // B[d] is filled with n m-1's
2. B[d] \(\leftarrow\) \(\{\)m-1,...,m-1\(\}\)
   // Each score is repeated n times in S
3. S \(\leftarrow\) \(\{\)m-2,...,m-2,m-3,...,m-3,...,1,...,1,0,...,0\(\}\)
4. \xwhile S \(\neq \{\}\)
     // The column of B that contains fewer than n elements, 
     // with the lowest sum
5.   c \(\leftarrow\) argmin\(_\)i\(\{\)sum(B[i]) + s(V,i) : |B[i]| < n\(\}\)
6.   B[c].push(S[0])
7.   S \(\leftarrow\) S - S[0]
8. \xif sum(B[d]) + s(V,d) \(\geq\) max\(_\)i\(\{\)sum(B[i]) + s(V,i)\(\}\)
9.   return B
10.\xelse
11.  return Failure
\end{program}
\caption{\label{fig:lslg} The greedy algorithm based on placing the largest remaining score in the column of A with the most room.}
\end{figure*}
\comment{used to be called fig:alg2}

The following proposition shows that this algorithm can sometimes find an optimal manipulation when $\zuck$ fails, and this is true for an infinite family of instances.
\begin{proposition}
\label{prop:1}
Let $E = (V, m)$ be an election such that $m >2$ is even, $d=m$, $s(V,d) = 0$ and $s(V, i) = \frac{m}{2}+i$ for all $i \neq d$.  Then $\lslg$ finds an optimal 2-manipulation, but $\zuck$ produces a 3-manipulation.
\end{proposition}

\begin{proof}
\end{proof}
First, note that two non-manipulator votes are always sufficient to create such an election.  Let $\sigma = <1,2,...,m-1>$ and let
\begin{displaymath}
\sigma' = <\frac{m}{2}+1, \frac{m}{2}+2,...,\frac{m}{2}+\frac{m}{2}-1,1,2,...,\frac{m}{2}>
\end{displaymath}

\noindent Then $\sigma + \sigma' = $

$$
<\left(1 + \frac{m}{2}+1\right), \left(2 + \frac{m}{2}+2\right),...,\left(\frac{m}{2}-1 + \frac{m}{2}+\frac{m}{2}-1\right), \left(\frac{m}{2}+1\right),...,\left(m-1 + \frac{m}{2}\right) >
$$

\noindent which gives us $\frac{m}{2} + 2x$ for $1\leq x \leq \frac{m}{2}-1$ and $\frac{m}{2}+2x-1$ for $1\leq x \leq \frac{m}{2}$, or in other words, $\frac{m}{2}+i$ for all $1\leq i \leq m-1$ (i.e. all $i \neq d$).

The first vote generated by $\zuck$ is $r_1 = m > 1 > 2 >...>m-1$, after which $s(V\cup \{r_1\},i) = \frac{m}{2}+m-1$ for all competing candidates, which is larger than the score of the distinguished candidate $s(V\cup \{r_1\},m) = m-1$.  Therefore another manipulator is added, without loss of generality its vote is $r_2 = m > 1 > 2 > ....m-1$.  The resulting scores of the competing candidates are $s(V\cup\{r_1,r_2 \},i) = \frac{m}{2}+ (m-1) + (m-i-1) = (5/2)m -2 -i$.  So candidate $i = 1$ still has larger score than $s(V\cup\{r_1,r_2 \},m) = 2m-2$.  Therefore, $\zuck$ does not find a 2-manipulation.

\begin{wrapfigure}{r}{0.5\textwidth}
 \centering
 Ê Ê\includegraphics[width=0.5\textwidth]{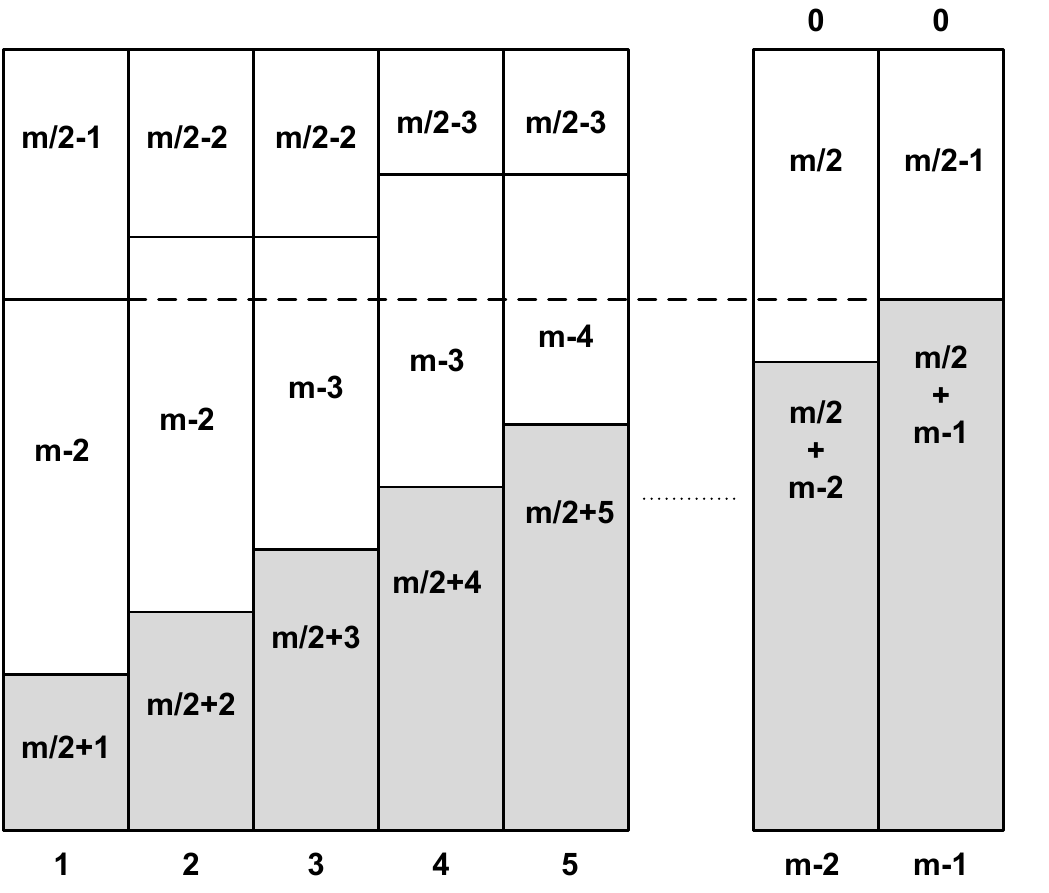}
 Ê Ê\caption{\label{fig:diagram} The 2-manipulation generated by $\lslg$ for the election in Proposition~\ref{prop:1} 
 Ê Ê}
\end{wrapfigure}
The first $m-1$ iterations of $\lslg$ will place the $k^{th}$ largest score from $S_2$ into the $k^{th}$ column of matrix $B$ for $1\leq k \leq m-1$. Note that the $k^{th}$ largest score is $m-2-\lfloor(k-1)/2\rfloor$.  Let $B_{m-1}$ be the matrix at this point.  Then $sum(B_{m-1}(i)) + s(V,i) = (m-2-\lfloor(i-1)/2\rfloor) + \frac{m}{2}+i$ for all $i < m$.  The next $m-1$ iterations of $\lslg$ will place the $k^{th}$ largest score from $S_2$ into the $k^{th}$ column of matrix $B$ for $m\leq k \leq 2(m-1)$. So column $i  <m$ will receive the element $\frac{m}{2}-1-\lceil(i-1)/2 \rceil$. Let $B_{2(m-1)}$ be the matrix when the loop terminates. Then $sum(B_{2(m-1)}(i)) + s(V,i) = (m-2-\lfloor(i-1)/2\rfloor) + (\frac{m}{2}+i) + (\frac{m}{2}-1-\lceil(i-1)/2 \rceil) = 2(m-1)$ for all $i < m$, while the achievable score of $m$ is also $2(m-1)$.  Therefore, $\lslg$ does find a 2-manipulation.  Figure~\ref{fig:diagram} shows the matrix generated by $\lslg$ (column $d=m$ is omitted), where the shaded areas represents the scores $s(V, i)$ for each $i < m$.

\hfill $\Box$

Unfortunately, $\lslg$ does not share the guarantee of $\zuck$ that in the worst case it requires one extra manipulator than is optimal.  In fact, Theorem~\ref{thm:alg2} shows that the number of extra manipulators $\lslg$ might require is unbounded.

\begin{theorem}
\label{thm:alg2}
Let $k$ be positive integer greater than zero and divisible by 36.  Let $s(V, 1) = 6k$, $s(V, 2) = 4k$, $s(V, 3) = 2k$, $s(V, 4) = 0$ be the scores of four candidates after some non-manipulators $V$ vote, and let $d=4$.  Then $\zuck$ will find the optimal manipulation, using $2k$ manipulators.  However, $\lslg$ requires at least $2k + k/9 - 3$ manipulators.
\end{theorem}
\begin{proof}
First, we should mention that for any $k$ there is a set of votes $V_k$ that gives the specified scores to the four candidates:  $V_k$ is simply $2k$ votes, all equal to $1 > 2 > 3 > 4$.  $\zuck$ will use $2k$ manipulators, all voting $4 > 3 > 2 > 1$, to achieve a score of $6k$ for all candidates (the only optimal manipulation). It remains to argue that $\lslg$ requires more than $2k + k/9 - 4$ manipulators.  Assume for contradiction that we find a manipulation using $n = 2k + k/9 -4 = 19k/9 -4$ manipulators.  We will follow the execution of $\lslg$ until a contradiction is obtained.   Note that given our definition of $n$, since $k$ is divisible by 4 and 9, $\frac{n-k}{2}$ is an integer.

First, the algorithm will place $k$ 2's in B[3], at which point $sum(B[3]) = 2k + 2k = 4k = s(V_k, 2)$.  Then it will begin to place 2's in columns B[2] and B[3] evenly, until all remaining $n-k$ 2's have been placed into B.  At this point, B[2] contains $\frac{n-k}{2}$ 2's, and the number of 2's that B[3] contains is $k + \frac{n-k}{2} = k/2 + n/2 = k/2 + (19k/9 - 4)/2 = 14k/9 - 2 < 19k/9 - 4 = n$. So at this point, B[3] is not full yet and B[2] isn't either (it has fewer elements than B[3]).  Both columns sum to $4k + 2(\frac{n-k}{2}) = 46k/9 - 4 = 5k + k/9 - 4 < 6k$.  Therefore, the algorithm will start putting 1's in both B[2] and B[3] evenly, until either their column sums reach $6k$ or B[3] gets filled.  In fact, B[3] will be filled before its sum reaches $6k$, since B[3] requires $\frac{n-k}{2}$ more elements to be filled, but at this point, $sum(B[2])=sum(B[3]) = 46k/9 - 4 + \frac{n-k}{2} = 51k/9 - 6 = 5k + 2k/3 - 6 < 6k$.  

Now, the algorithm will continue by putting $k/3 + 6$ 1's into B[2], at which point $sum(B[2]) = 51k/9 - 6 + k/3 + 6 = 6k$. Then the algorithm will start putting 1's evenly in both B[1] and B[2], until either it runs out of 1's or B[2] is filled.  In fact, the 1's will run out before B[2] is filled, since B[2] requires $n - (\frac{n-k}{2} + \frac{n-k}{2} + k/3 + 6) = 2k/3 - 6$ more elements, which is equal to the number of remaining 1's, but these are spread between B[1] and B[2].  So B[2] will get $(2k/3 - 6)/2 = k/3 - 3$ additional 1's, for a total of $sum(B[2]) = 4k + 2(\frac{n-k}{2}) + \frac{n-k}{2} + k/3 + 6 + k/3 -3 = 19k/3 - 3 > 19k/3 -12 = 3n$.  Since $sum(B[2]) > 3n$ there is no manipulation using $n = 19k/9 - 4$ manipulators.  Therefore, $\lslg$ requires at least $n+1 = 2k + k/9 -3$ manipulators.
\hfill $\Box$
\end{proof}

This result shows the weakness of $\lslg$, that it only considers the relative sizes of the competing candidates' current scores.  Therefore 
if two candidates' column sums ever become equal during $\lslg$, they will often be treated equivalently for the remainder of the iterations.  In the example from Theorem~\ref{thm:alg2}, this is the fatal mistake, since at the point where $sum(B[3])$ becomes equal to  $sum(B[2])$, column $3$ requires fewer additional elements before it is filled (i.e. $|B[2]| < |B[3]|$).  Therefore, it is important for column 3 to receive larger elements than column 2.  In fact, all of the largest elements must be taken by column $3$, and none given to column $2$.  However, $\lslg$ will begin treating the two equal columns the same, distributing the remaining 2's evenly between $B[2]$ and $B[3]$.  This observation motivates our second greedy algorihthm.

\subsection{Average Desired Score}
\begin{figure*}
\begin{program}
LSLA(V,n)
1. B[i] \(\leftarrow \emptyset\) for all 1\(\leq\)i\(\leq\)m
   // B[d] is filled with n m-1's
2. B[d] \(\leftarrow\) \(\{\)m-1,...,m-1\(\}\)
   // Each score is repeated n times in S
3. S \(\leftarrow\) \(\{\)m-2,...,m-2,m-3,...,m-3,...,1,...,1,0,...,0\(\}\)
4. \xwhile S \(\neq \{\}\)
     // The column of B with highest average desired score
5.   c \(\leftarrow\) argmax\(_\)i\(\{\) [g_i-sum(B[i])] / [n-|B[i]|]] : |B[i]| < n\(\}\)
6.   s \(\leftarrow\) chooseScore(g_c-sum(B[c]), S)
7.   B[c].push(s)
8.   S \(\leftarrow\) S - \(\{\)s\(\}\)
9. \xif sum(B[d]) + s(V,d) \(\geq\) max\(_\)i\(\{\)sum(B[i]) + s(V,i)\(\}\)
10.  return B
11.\xelse
12.  return Failure
\end{program}

\begin{program}
chooseScore(g,S)
1. s \(\leftarrow\) max\(\{\)s \(\in\) S : s \(\leq\) g\(\}\)
2. \xif s = None
3.    s = S[0]
4. return s  
\end{program}
\caption{\label{fig:avgalg} The greedy algorithm based on average desired score, for n manipulators.}
\end{figure*}

The second greedy algorithm is based on the idea that it is not enough to simply assign the largest scores to the columns of $B$ that have the largest gap.  Each column of $B$ also requires exactly $n$ elements in order to be filled, where $n$ is the number of manipulators currently attempted. To balance these two requirements, we can look at the \emph{remaining gap} $g_i - sum(B[i])$ and divide it by the remaining number of scores that must be added to column $i$, $n - |B[i]|$.  Notice that if we had $n - |B[i]|$ scores of this average size available (for each $i$), we could fill every column of $B$ perfectly.  Since we don't, a sensible heuristic is to put the largest scores in the columns that have largest average desired score.   This algorithm, called $\lsla$, is shown in Figure~\ref{fig:avgalg}. \\

The structure of $\lsla$ is similar to $\lslg$, so it will not be explained line by line.  Note that on line 5 of $\lsla$ we need some way to break ties between candidates that have the same average desired score.  We could break ties arbitrarily, but we also consider choosing the candidate $i$ with minimum $|B[i]|$ since this column needs more additional scores.  We found experimentally that the latter tie breaking policy works better overall, although there are some instances where only the arbitrary policy finds the optimal manipulation. The procedure $\texttt{chooseScore}$ is used to avoid violating the maximum column sum $g_i$ earlier than necessary. Given an array of unassigned scores and the size of a column's remaining gap, it returns the largest unassigned score that fits in the remaining gap. We found experimentally that this was vital to finding the optimal manipulation in the majority of cases.

We now compare $\lsla$ to the other two greedy algorithms. $\lsla$ behaves similarly to $\zuck$ on the instances from Theorem~\ref{thm:alg2}, and thus it performs better than $\lslg$ on an infinite family of instances.  In fact, in the next section we will see that we have never found an instance for which $\zuck$ can find an optimal manipulation but $\lsla$ fails.  However, cases do exist where the simpler greedy algorithm $\lslg$ finds the optimal manipulation and $\lsla$ fails.  Two examples are shown in Figure~\ref{fig:exavgfails}, but analysis of these cases has failed to produce a generalizable pattern.  In the next section we provide further experimental evidence of the superiority of $\lsla$ compared to the other two algorithms.

\begin{figure*}
\centering{
\begin{tabular}{r|cccccccc}
$j$ & 1 & 2 & 3 & 4 & 5 & 6 & 7 & 8\\
\hline
$S(V_1, j)$ & 67& 60& 59& 58& 58& 52& 52& 42\\
$S(V_2, j)$ & 41& 34& 30& 27& 27& 26& 25& 14\\
\end{tabular}}
\caption{ \label{fig:exavgfails} Examples where $\lslg$ beats $\lsla$ by finding the optimal number of manipulators vs. using one extra.}
\end{figure*}

\section{Empirical Comparison}
\label{sec:exper}
In this section we compare the performance of $\zuck$, $\lslg$ and $\lsla$ from a practical perspective.  Our experimental setup is similar to that of Walsh~\cite{walshECAI10}.  We consider two methods of generating non-manipulator votes, the uniform random votes model and the Polya Eggenberger urn model~\cite{bergPC85}.  In the uniform random votes model, each vote is drawn uniformly at random from all $m!$ possible votes.  In the urn model, votes are drawn from an urn at random, but we place them back into the urn along with $a$ other votes of the same type.  This model attempts to capture varying degrees of social homogeneity, or the similarity between voters' preferences. We set $a = m!$, which means that there is a 50\% chance that the second vote is the same as the first.  It would be interesting to consider varying the degree of vote similarity by experimenting with different values of $a$.  In future work we also intend to study votes generated from real-world elections, e.g.~\cite{dobra83}. We generated election instances for numbers of candidates $m$ and numbers of non-manipulators $p$ in $\{2^2,...,2^7\}$.  We generated 1000 instances for each pair $(m, p)$.  Since the votes were generated randomly, for small numbers of candidates some duplicate instances were produced.  The total number of distinct Uniform elections obtained was 32679, and the number of distinct Urn elections was 31530.\\

In order to determine the optimal number of manipulators exactly, we modeled the manipulation problem as a constraint satisfaction problem (CSP).  The model we used comes directly from the definition of the manipulation matrix $A$, Definition~\ref{def:matrix}.  In this model, there are $n\times m-1$ finite domain variables, with domains equal to $\{0,...,m-2\}$ that represent the unknown elements of $A$.  There are $n$ ALLDIFF constraints, each over the variables of a row, that ensure each vote is properly formed. $m-1$ constraints over the variables of each column $i$ of $A$ ensure that their sum is at most $g_i$.  Finally, if $g_i = g_j$ for any two columns $i < j$, we added a constraint that $A[i][0] < A[j][0]$ over their row-1 elements.  This breaks the symmetry between the two columns and reduces the number of equivalent solutions to the model. We used the solver Gecode~\cite{gecode} to find a solution to the CSP, using Domain Over Weighted Degree as the variable ordering heuristic.  The timeout for Gecode was set to one hour, and all experiments were performed on processors of typical contemporary performance.\\

We will refer to the number of manipulators used by $\zuck$ as $N_r$. We ran the three competing greedy algorithms, and if this did not determine the optimal manipulation (i.e. none did better than $\zuck$), we checked whether Observation~\ref{obs:1} or the fact that $g_{E,N_r-1}(i)$ is negative for some candidate $i$ allow us to conclude that a $(N_r - 1)$-manipulation is impossible.  If the optimal number of manipulators was still unknown, we attempted to find an $(N_r-1)$-manipulation using Gecode. \\

\begin{figure*}
\centering{
\begin{tabular}{rr|cccc}
 $m$  &  \# Inst.  & $\zuck$ & $\lslg$  &  $\lsla$ & $\lslg$ beat $\lsla$\\ 
\hline
4 & 2771 & 2611 & 2573 & 2771 & 0 \\
8 & 5893 & 5040 & 5171 & 5852 & 2\\
16 & 5966 & 4579 & 4889 & 5883 & 3\\
32 & 5968 & 4243 & 4817 & 5879 & 1\\
64 &  5962 & 3980 & 4772 & 5864 & 3\\
128 & 5942 & 3897 & 4747 & 5821 & 2\\
\hline
Total & 32502 & 24350 & 26969 & 32070 & 11 \\
\% & & 75 & 83 & 99 & $<$1
\end{tabular}}
\caption{ \label{fig:table1} Number of Uniform elections for which each algorithm found an optimal manipulation.}
\end{figure*}

\begin{figure*}
\centering{
\begin{tabular}{rr|cccc}
 $m$  &  \# Inst.  & $\zuck$ & $\lslg$  &  $\lsla$ & $\lslg$ beat $\lsla$\\ 
\hline
4 & 3929 & 3666 & 2604 & 3929 & 0 \\
8 & 5501 & 4709 & 2755 & 5496 & 0\\
16 & 5502 & 4357 & 2264 & 5477 & 1\\
32 & 5532 & 4004 & 2008 & 5504 & 0\\
64 &  5494 & 3712 & 1815 & 5475 & 0\\
128 & 5571 & 3593 & 1704 & 5565 & 0\\
\hline
Total & 31529 & 24041 & 13150 & 31446 & 1 \\
\% & & 76 & 42 & 99.7 & $<$1
\end{tabular}}
\caption{ \label{fig:table2} Number of Urn elections for which each algorithm found an optimal manipulation.}
\end{figure*}

\noindent \emph{Uniform Elections} Using the combined method described above, we were able to determine the optimal number of manipulators in 32502 out of the 32679 distinct Uniform elections.  The results are shown in Figure~\ref{fig:table1}, grouped by the number of candidates $m$.  The first column shows the number of candidates, and the second column shows the number of instances for which we report results.  The next three columns show the number of instances for which each of the greedy algorithms could find an optimal manipulation.  The last column shows the number of instances on which $\lslg$ found the optimal solution but $\lsla$ did not. These results show that both $\lslg$ and $\lsla$ provide a significant improvement over $\zuck$, solving 83\% and 99\% of instances to optimality overall. We also notice that $\zuck$ solves fewer problems to optimality as the number of candidates increases, while $\lsla$ does not seem to suffer from this problem as much: $\lsla$ solves 100\% of the $m=4$ instances and 98\% of the 128 candidate elections. In addition to the results in the table, we mention that in every one of the 32502 instances, if $\zuck$ found an $n$-manipulation either $\lsla$ did too, or $\lsla$ found an $(n-1)$-manipulation.\\

\noindent \emph{Urn Elections} We were able to determine the optimal number of manipulators for 31529 out of the 31530 unique Urn elections.  Figure~\ref{fig:table2} presents the results, in the same format as Figure~\ref{fig:table1}.  $\zuck$ solves about the same proportion of the Urn instances as it did of the Uniform instances, 76\%.  However, $\lslg$ performance drops significantly, and is in fact much worse than $\zuck$ at 42\% of instances solved.  This can be explained by the structure of the Urn elections, which contain many identical votes.   This results in a similar pattern of non-manipulator scores to those in Theorem~\ref{thm:alg2} on which $\lslg$ has pathological behavior.  Surprisingly, the good performance of $\lsla$ is maintained.  $\lsla$ found the optimal manipulation on more than 99\% of the instances, dominates $\zuck$ and only lost one instance to $\lslg$ in this set of experiments.\\

\vspace{-0.5cm}
\section{Conclusion}
\label{sec:conclusion}
We studied the coalitional manipulation problem in elections using the
unweighted Borda rule. We provided insight into the structure of the
solutions that allows us to build algorithms that construct a
manipulation in a manner similar to bin-packing rather than
constructing an entire vote at each step. Using this insight, we
proposed two new algorithms, $\lslg$ and $\lsla$. We have provided
no optimality guarantees for these algorithms. In fact, we show that
$\lslg$ may require an unbounded number of additional manipulators
relative to the optimal. However, there are infinite families of instances in which both
algorithms can find the optimal but the algorithm proposed by Zuckerman
et al.~\cite{zuckermanSODA08}, which does have a worst-case guarantee, can not.  In an
empirical evaluation performed over more than 60000 randomly generated
instances, $\lsla$ finds the optimal manipulation in more than 99\% of
the cases, is never outperformed by $\zuck$ and in only 12 instances by
$\lslg$. This result provides further empirical evidence that the
unweighted Borda rule can be manipulated effectively using relatively
simple algorithms.\\

In future work, we intend to determine whether we can provide
theoretical optimality guarantees for $\lsla$ similar to those that are
known for $\zuck$ and theoretically verify the strict dominance that we
observed empirically. Further, we intend to investigate whether we can
extend our algorithms to always find the optimal number of
manipulators for these elections. Another question that arises
from this work is whether similar insights can be developed for other
scoring rules.\\

\noindent {\footnotesize \textbf{Acknowledgments:} Jessica Davies is supported by the National Research Council of Canada. George Katsirelos is supported by the ANR UNLOC project ANR 08-BLAN-0289-01. Nina Narodytska and Toby Walsh are funded by the Australian Government's Department of Broadband, Communications and the Digital
Economy and the Australian Research Council.}

\begin{minipage}[n]{0.5\linewidth}

\begin{contact}
Jessica Davies\\
University of Toronto\\
Toronto, Canada\\
\email{jdavies@cs.toronto.edu}
\end{contact}

\begin{contact}
George Katsirelos\\
Universit\'e Lille-Nord de France\\
CRIL/CNRS UMR8188\\
Lens, F-62307, France\\
\email{gkatsi@gmail.com}
\end{contact}

\end{minipage}
\begin{minipage}[ b]{0.5\linewidth}
\begin{contact}
Nina Narodytska\\
NICTA and UNSW\\
Sydney, Australia\\
\email{ninan@cse.unsw.edu.au}
\end{contact}

\begin{contact}
Toby Walsh\\
NICTA and UNSW\\
Sydney, Australia\\
\email{toby.walsh@nicta.com.au}
\end{contact}
\end{minipage}


\begin{thebibliography}{10}

\bibitem{bergPC85}
S.~Berg. Paradox of Voting Under an Urn Model: The Effect of Homogeneity. \emph{Public Choice}, 47:377-387, 1985.

\bibitem{black58}
D.~Black. \emph{The Theory of Committees and Elections}. Cambridge University Press, 1958.

\bibitem{brelsfordAAAI08}
E.~Brelsford, P.~Faliszewski, E.~Hemaspaandra, H. Schnoor and I.~Schnoor. 
Approximability of Manipulating Elections. 
In \emph{Proc.\ 23rd Conference on Artificial Intelligence (AAAI-2008)}, p.44-49, 2008.

\bibitem{coffman06}
E.~Coffman, J.~Csirik and J.~Leung. Variants of Classical One-Dimensional Bin Packing. Chapter 33 of \emph{Handbook of Approximation Algorithms and Meta-Heuristics}, Teofilo Gonzalez, ed., Francis and Taylor Books (CRC Press), 2006.

\bibitem{csaaai2006}
V.~Conitzer and T.~Sandholm. Nonexistence of Voting Rules That Are Usually Hard to Manipulate. In \emph{Proc. of the 21st National Conference on AI (AAAI-06)},2006.

\bibitem{conitzerJACM08}
V.~Conitzer, T.~Sandholm and J.~Lang. 
When are Elections with Few Candidates Hard to Manipulate?. 
In \emph{Journal of the Association for Computing Machinery}, 54, 2007.

\bibitem{dobra83}
J.~Dobra. An Approach to Empirical Studies of Voting Paradoxes: An Update and Extension. \emph{Public Choice}, 41:241-250, 1983.

\bibitem{hall}
P.~Hall. On Representatives of Subsets. \emph{Journal of the London Mathematical Society}, 10:26-30, 1935.

\bibitem{krauseJACM75}
K.~Krause, V.~Shen and H.~Schwetman. 
Analysis of Several Task-Scheduling Algorithms for a Model of Multiprogramming Computer Systems. 
\emph{Journal of the Association for Computing Machinery}, 22(4):522-550, 1975.





\bibitem{prjair07}
A.~Procaccia and J.~Rosenschein. Junta Distributions and the Average-Case Complexity of Manipulating Elections. \emph{Journal of Artificial Intelligence Research}, 28:157-181, 2007.

\bibitem{reillyIPR02}
B.~Reilly. Social Choice in the South Seas: Electoral Innovation and the Borda Count in the Pacific Island Countries. \emph{International Political Review}, 23(4):355-372, 2002.

\bibitem{gecode}
C.~Schulte, M.~Lagerkvist and G.~Tack. Gecode. http://www.gecode.org/

\bibitem{walshIJCAI09}
T.~Walsh. Where are the Really Hard Manipulation Problems? The Phase Transition in
Manipulating the Veto Rule. In \emph{Proc. of IJCAI-2009}, 2009.

\bibitem{walshECAI10}
T.~Walsh. An Empirical Study of the Manipulability of Single Transferable Voting. \emph{to
appear in Proc. of ECAI 2010}.

\bibitem{xcec08}
L.~Xia and V.~Conitzer. Generalized Scoring Rules and the Frequency of Coalitional Manipulability. In \emph{Proc. of the 9th ACM Conference on Electronic Commerce (EC-08)}, p.109-118, 2008.

\bibitem{xcec08b}
L.~Xia and V.~Conitzer. A Sufficient Condition for Voting Rules to Be Frequently Manipulable. In \emph{Proc. of the 9th ACM Conference on Electronic Commerce (EC-08)}, p.99-108, 2008.

\bibitem{xiaEC10}
L.~Xia, V.~Conitzer and A.~Procaccia. 
A Scheduling Approach to Coalitional Manipulation. 
In \emph{Proc.\ 11th ACM Conference on Electronic Commerce}, 2010.



\bibitem{xiaIJCAI09}
L.~Xia, M.~Zuckerman, A.~Procaccia, V.~Conitzer and J.~Rosenschein. 
Complexity of Unweighted Coalitional Manipulation Under Some Common Voting Rules. 
In \emph{Proc.\ 21st International Joint Conference on Artificial Intelligence (IJCAI-2009)}, p.348-353, 2009.

\bibitem{zuckermanSODA08}
M.~Zuckerman, A.~Procaccia and J.~Rosenschein. 
Algorithms for the Coalitional Manipulation Problem. 
In \emph{Proc.\ 19th Symposium on Discrete Algorithms (SODA-2008)}, p.277-286, 2008.
\end{thebibliography}
\end{document}